\documentclass{amsart}
\title{Convergence of Expected Utility For Universal AI}
\author{Peter de Blanc\\Department of Mathematics\\Temple University}
\thanks{Written in association with the Singularity Institute for Artificial Intelligence.}
\date{\today}

\usepackage{algorithmic}

\newtheorem{unbounded_above}{Definition}

\newtheorem{heaven_finder}{Lemma}
\newtheorem{bb}[heaven_finder]{Lemma}

\newtheorem{cor_above}{Corollary}
\newtheorem{cor_both}[cor_above]{Corollary}

\begin{document}
\begin{titlepage}
\maketitle
\end{titlepage}
\section{Abstract}

We consider a sequence of repeated interactions between an agent and an environment. Uncertainty about the environment is captured by a probability distribution over a space of hypotheses, which includes all computable functions. Given a utility function, we can evaluate the expected utility of any computational policy for interaction with the environment. After making some plausible assumptions (and maybe one not-so-plausible assumption), we show that if the utility function is unbounded, then the expected utility of any policy is undefined.

\section{AI Formalism}

We will assume that the interaction between the agent and the environment takes place in discrete time-steps, or cycles. In cycle $n$, the agent outputs an action $y_n \in Y$, and the environment inputs to the agent a perception $x_n \in X$. $Y$ and $X$ are the sets of possible actions and perceptions, respectively, and are considered as subsets of $\mathbb{N}$. Thus the story of all interaction between agent and environment is captured by the two sequences $x = \left(x_1, x_2, \dots\right)$ and $y = \left(y_1, y_2, \dots\right)$.

Let us introduce a notation for substrings. If $s$ is a sequence or string, and $\{a, b\} \subseteq \mathbb{N}$, $a \le b$, then define $s_a^b = \left(s_a, s_{a+1}, \dots s_b\right)$.

We will denote the function instantiated by the environment as $Q: Y^* \rightarrow X$, so that $\forall n \in \mathbb{N}$, $x_n = Q(y_1^n)$. This means that the perception generated by the environment at any given cycle is determined by the agent's actions on that and all previous cycles.

A policy for the agent is a function $p: \left(Y^* \times X^*\right) \rightarrow Y$, so that an agent implementing $p$ at time $n$ will choose an action  $y_n = p\left(y_1^{n-1}, x_1^{n-1}\right)$.

If, at any time, an agent adopts some policy $p$, and continues to follow that policy forever, then $p$ and $Q$ taken together completely determine the future of the sequences $\left(x_n\right)$ and $\left(y_n\right)$. We are particularly interested in the future sequence of perceptions, so we will define a \emph{future function} $\Psi\left(Q, p, y_1^n, x_1^n\right) = x_{n+1}^\infty$.

Because the precise nature of the environment $Q$ is unknown to the agent, we will let $\Omega$ be the set of possible environments. Let $F$ be a $\sigma$-algebra on $\Omega$, and $P: F \rightarrow [0, 1]$ be a probability measure on $F$ which represents the agent's prior information about the environment.

We will also define a function $\Gamma_q: Y^* \rightarrow X^*$ which represents the perception string output by environment $q$ given some action string. Let \newline $\Gamma_q\left(s\right) = \left(q\left(s_1^1\right), q\left(s_1^2\right), \dots, q\left(s\right)\right)$.

The agent will compare the quality of different outcomes using a utility function $U: X^\mathbb{N} \rightarrow \mathbb{R}$. We can then judge a policy by calculating the expected utility of the outcome given that policy, which can be written as

\begin{equation}
	E\left( U\left(x_1^n \Psi\left(Q, p, y_1^n, x_1^n\right)\right) | \Gamma_Q(y_1^n) = x_1^n \right)
\end{equation}

...where Q is being treated as a random variable. When we write a string next to a sequence, as in $x_1^n \Psi\left(Q, p, y_1^n, x_1^n\right)$, we mean to concatenate them. Here, $x_1^n$ represents what the agent has seen in the past, and $\Psi\left(Q, p, y_1^n, x_1^n\right)$ represents something the agent may see in the future. By concatenating them, we get a complete sequence of perceptions, which is the input required by the utility function $U$.

Notice that the expected utility above is a conditional expectation. Except on the very first time-step, the agent will already have some knowledge about the environment. After $n$ cycles, the agent has output the string $y_1^n$, and the environment has output the string $x_1^n$. Thus the agent's knowledge is given by the equation $\Gamma_Q\left(y_1^n\right) = x_1^n$.

Agents such as AIXI (Hutter, 2007) choose actions by comparing the expected utility of different policies. Thus we will focus, in this paper, on calculating the expected utility of a policy.

\section{Assumptions about the Hypothesis Space}

Here we'll make further assumptions about the hypothesis space $\left(\Omega, F, P\right)$. While we could succinctly make strong assumptions that would justify our central claim, we will instead try to give somewhat weaker assumptions, even at the loss of some brevity.

Let $\Omega_C$ be the set of computable total functions mapping $Y^*$ to $X$. We will assume that $\Omega \supseteq \Omega_C$ and that $\left(\forall q \in \Omega_C\right): \{q\} \in F$ and $P\left(\{q\}\right) > 0$. Thus we assume that the agent assigns a nonzero probability to any computable environment function.

Let $\Omega_P$ be the set of computable partial functions from $Y^*$ to $X$. Then $\Omega_C \subset \Omega_P$. The computable partial functions $\Omega_P$ can be indexed by natural numbers, using a surjective computable index function $\phi: \mathbb{N} \rightarrow \Omega_P$. Since the codomain of $\phi$ is a set of partial functions, it may be unclear what we mean when we say that $\phi$ is computable. We mean that $\left(i, s\right) \rightarrow \left(\phi\left(i\right)\right)\left(s\right)$, whose codomain is $X$, is a computable partial function. We will also use the notation $\phi_i = \phi\left(i\right)$.

We'll now assume that there exists a computable total function $\rho: \mathbb{N} \rightarrow \mathbb{Q}$ such that if $\phi_i \in \Omega_C$, then $0 < \rho\left(i\right) \le P\left(\{\phi\left(i\right)\}\right)$. Intuitively, we are saying that $\phi$ is a way of describing computable functions using some sort of language, and that $\rho$ is a way of specifying lower bounds on probabilities based on these descriptions. Note that we make no assumption about $\rho\left(i\right)$ when $\phi_i \notin \Omega_C$.

To see an example of a hypothesis space satisfying all of our assumptions, let $\Omega = \Omega_C$, let $F = 2^{\Omega_P}$, let $\phi$ be any programming language, and let $\rho\left(i\right) = 2^{-i}$. Let
\begin{equation}
	O = \sum_{i \ni \left(\phi_i \in \Omega_C\right)} \rho\left(i\right)
\end{equation}
and for any $\omega \in \Omega$, let
\begin{equation}
	P \left( \{ \omega \} \right) = \frac{1}{O} \sum_{i \ni \left(\phi_i = \omega\right)} \rho\left(i\right)
\end{equation}

\section{Assumptions about the Utility Function}

Perhaps the most philosophically questionable assumption in this paper has already been made in defining the domain of the utility function $U$ as $X^\mathbb{N}$, the set of perception-sequences. This is like assuming that a person cares not about his or her family and friends, but about his or her perception of his or her family and friends.

Since the utility function $U: X^\mathbb{N} \rightarrow \mathbb{R}$ takes as its argument an infinite sequence, we must discuss what it means for such a function to be computable. Obviously any computation which terminates can only look at a finite number of terms. Therefore we will try to approximate $U\left(x\right)$ using prefixes of $x$. We say that $U$ is computable if there exist computable functions $U_L, U_U: X^* \rightarrow \mathbb{Q} \cup \left\{-\infty, +\infty\right\}$ such that, if $x \in X^\mathbb{N}$ and $\bar{x} \in X^*$ and $\bar{x} \sqsubseteq x$, then:

\begin{itemize}
\item $U_L\left(\bar{x}\right) \le U(x) \le U_U\left(\bar{x}\right)$
\item $U_L\left(\bar{x}\right) \rightarrow U\left(x\right)$ and $U_U\left(\bar{x}\right) \rightarrow U\left(x\right)$ as $\bar{x} \rightarrow x$.
\end{itemize}

In any case, we will \emph{not} assume that $U$ is computable, because we do not need such a strong assumption to prove our claims. Instead we will define two possible conditions.

\begin{unbounded_above}
	\label{un_abv}
	Let $D \subseteq X^\mathbb{N}$ and let $U: D \rightarrow \mathbb{R}$. Let $D^p = \{s \in X^* | \left(\exists d \in D\right): s \sqsubseteq d\}$. Then $U$ is \emph{computably unbounded from above} on $D$ if there exists a computable partial function $U_L: D^p \rightarrow \mathbb{Z}$ such that:
	\begin{itemize}
	\item $\left(\forall d \in D\right) \left(\forall s \in D^p\right):$ if $s \sqsubseteq d$, and if $U_L\left(s\right)$ exists, then $U_L\left(s\right) \le U\left(d\right)$.
	\item $\left(\forall m \in \mathbb{Z}\right) \left(\exists s \in D^p\right): U_L\left(s\right) > m$.
	\end{itemize}
\end{unbounded_above}

$U$ is \emph{computably unbounded from below} if $-U$ is computably unbounded from above.

Note in particular that any computable function on $X^\mathbb{N}$ which is unbounded from above is computably unbounded from above, and any computable function which is unbounded from below is computably unbounded from below.

The following lemma will help us find environments which generate large amounts of utility. When considering $f$ in the lemma, think of $U_L$ above.

\begin{heaven_finder}
	\label{heav}
	Suppose $C \subseteq X^*$, and $f: C \rightarrow \mathbb{Z}$ is a computable partial function such that $\left(\forall m \in \mathbb{Z}\right) \left(\exists c \in C\right): f\left(c\right) > m$. Then there exists a computable total function $H: \mathbb{Z} \rightarrow C$ such that, $\left(\forall m \in \mathbb{Z}\right): f \circ H\left(m\right) \ge m$.
\end{heaven_finder}

In other words, given an \emph{unbounded} partial function $f$, there is a computable function $H$ which finds an input on which $f$ will exceed any given bound.

\begin{proof}
	First we'll index $C$; let $C = \left\{c_1, c_2, \dots\right\}$.

	If $f$ were a total function, we could simply let $H\left(m\right) = c_{\min \left\{i \in \mathbb{N} : f(c_i) > m\right\}}$. We would compute this by first computing $f(c_1)$, then $f(c_2)$, etc. Unfortunately we only have that $f$ is a partial function, so we can not proceed in this way.

	Instead, we'll note that for any input on which $f$ halts, it must halt in a specific number of steps. The Cantor pairing function $\pi: \mathbb{N} \times \mathbb{N} \rightarrow \mathbb{N}$, $\pi\left(k_1, k_2\right) = \frac{1}{2} (k_1 + k_2) (k_1 + k_2 + 1) + k_2$ is a bijection, so we can use $\pi^{-1}$ to index all pairs of natural numbers. Then we can simulate $f$ on every possible input for every number of steps, which will allow us to evaluate $f$ on every input for which $f$ halts.

\begin{verbatim}
	def H(m):
	    for n in (1, 2, 3, ...):
	        let (t, i) = pi^-1(n)
	        simulate f(c_i) for t steps
	            ... if it does not finish:
	                do nothing.
	            ... if it does finish:
	                if f(c_i) >= m:
	                    return c_i
\end{verbatim}

	Then $H$ is a computable total function and $f \circ H\left(m\right) \ge m$. 
\end{proof}

\section{Results}

Let $R$ be the set of all computable partial functions mapping $\mathbb{N}$ to $\mathbb{N}$, and let $\theta: \mathbb{N} \rightarrow R$ be a computable index (analogous to our other index function $\phi$).

Let
\begin{equation}
	\label{B_def}
	B(n) = \max_{k \le n} \theta_k\left(0\right)
\end{equation}

\begin{bb}
	\label{bb_lemma}
	Let $f \in R$ be a total function. Then $B(n) > f(n)$ infinitely often.
\end{bb}

\begin{proof}
	Suppose not. Then $B(n) > f(n)$ only finitely many times, so there exists some $c \in \mathbb{N}$ such that $(\forall n \in N): f(n) + c > B(n)$.

	Let $C(n, m) = f(n) + c$. By a corollary of the Recursion Theorem, there exists $m \in \mathbb{N}$ such that $(\forall n \in \mathbb{N}): \theta_m(n) = C(m, n) = f(m) + c$.

	By definition, $B(m) \ge \theta_m(0) = C(m,0) = f(m) + c > B(m)$. So $B(m) > B(m)$, which is a contradiction.
\end{proof}

Now suppose that at time $n+1$, the agent has already taken actions $y_1^n$ and made observations $x_1^n$, and is considering the expected utility of policy $p$. Let $D = \{s \in X^\mathbb{N} : s_1^n = x_1^n\}$.

\newtheorem{main_thm}{Theorem}
\begin{main_thm}
	\label{big_thm}
	If $U$ is computably unbounded from above on $D$, then \newline $E\left( U\left(x_1^n \Psi\left(Q, p, y_1^n, x_1^n\right)\right) | \Gamma_Q(y_1^n) = x_1^n \right)$ is either undefined or $+\infty$.
\end{main_thm}

\begin{proof}
	Let $U_L: D^p \rightarrow \mathbb{Z}$ be as in definition \ref{un_abv}. Then by Lemma \ref{heav}, there exists $H: \mathbb{Z} \rightarrow D^p$ such that $\left(\forall m \in \mathbb{Z}\right): U_L(H(m)) > m$.

	$H$ here is intended to be used to construct sequences with high utility. Since $H$ outputs a string rather than a sequence, we will pad it to get a sequence. Let $c \in X$ be some arbitrary word in the perception alphabet. Then let $\bar{H}: \mathbb{Z} \rightarrow D$, where $\bar{H}(n)$ is a sequence beginning with $H(n)$, followed by $c, c, c, \dots$.

	For brevity, let $W_p\left(q\right) = x_1^n \Psi\left(q, p, y_1^n, x_1^n\right)$. $W_p\left(q\right)$ represents the complete sequence of perceptions received by the agent, assuming that it continues to implement policy $p$ in environment $q$.

	We will now break up the expected utility into two terms, depending on whether or not $Q \in \Omega_C$.

	\begin{eqnarray*}
		E\left( U\left(W_p\left(Q\right)\right) | \Gamma_Q(y_1^n) = x_1^n \right) \\
		= P(Q \in \Omega_C) E\left( U\left(W_p\left(Q\right)\right) | \Gamma_Q(y_1^n) = x_1^n, Q \in \Omega_C \right) \\
		+ P(Q \notin \Omega_C) E\left( U\left(W_p\left(Q\right)\right) | \Gamma_Q(y_1^n) = x_1^n, Q \notin \Omega_C \right)
\\
		= \sum_{q \in \Omega_C} U\left(W_p\left(q\right)\right) P\left(\{q\} | \Gamma_Q(y_1^n) = x_1^n\right) 
\\ + P(Q \notin \Omega_C) E\left( U\left(W_p\left(Q\right)\right) | \Gamma_Q(y_1^n) = x_1^n, Q \notin \Omega
_C \right)
	\end{eqnarray*}

	We will show that the series:

	\begin{equation*}
		\sum_{q \in \Omega_C} U\left(W_p\left(q\right)\right) P\left(\{q\} | \Gamma_Q(y_1^n) = x_1^n\right)
	\end{equation*}

	has infinitely many terms $\ge 1$. We will do this by finding a sequence of environments whose utilities grows very quickly - more quickly than their probabilities can shrink.

	By equation \ref{B_def}, for each $j \in \mathbb{N}$ there exists $u_j \in \mathbb{N}$ such that $u_j \le j$ and $\theta_{u_j}\left(0\right) = B\left(j\right)$.

	Now we define a map on function indices $G: \mathbb{N} \rightarrow \mathbb{N}$ such that:

	\begin{displaymath}
		\phi_{G\left(n\right)}\left(\gamma\right) = H(\theta_n(0))_{|\gamma|}
	\end{displaymath}

	So $G$ takes the $\theta$-index of an $\mathbb{N} \rightarrow \mathbb{N}$ function (say, $g$), and returns the $\phi$-index of an environment which is compatible with all the data so far, and which is guaranteed to produce utility greater than $g(0)$. We can assume that $G$ is a computable function.

	So our sequence of environments will be $\left\{\phi_{G(u_j)}\right\}_{j=1}^\infty$.

	Then $U\left(W_p(\phi_{G(u_j)})\right) \ge B(j)$.
	Now let
	\begin{equation}
		\bar{\rho}(j) = \lceil \max_{k \le j} \frac{1}{\rho(G(k))} \rceil
	\end{equation}
	Then $\bar{\rho}$ is a computable, nondecreasing function. Since $\bar{\rho}$ is computable, $B(j) \ge \bar{\rho}(j)$ infinitely often. Since $u_j \le j$, then by definition, $\bar{\rho}(j) \ge \frac{1}{\rho(G(u_j))} \ge \frac{1}{P(\{\phi_{G(u_j)}\})}$.
	$P(\Gamma_Q(y_1^n) = x_1^n | Q = \phi_{G(u_j)}) = 1$, so by Bayes' Rule, $P(\{\phi_{G(u_j)}\} | \Gamma_Q(y_1^n) = x_1^n) \ge P(\{\phi_{G(u_j)}\})$. Since both sides are positive, we take the reciprocal to get $\frac{1}{P(\{\phi_{G(u_j)}\} | y_1^n, x_1^n)} \le \frac{1}{P(\{\phi_{G(u_j)}\})}$.
	By transitivity, $U\left(W_p(\phi_{G(u_j)})\right) \ge \frac{1}{P(\{\phi_{G(u_j)}\} | \Gamma_Q(y_1^n) = x_1^n)}$ infinitely often, so $U\left(W_p(\phi_{G(u_j)})\right) P(\{\phi_{G(u_j)}\} | \Gamma_Q(y_1^n) = x_1^n) \ge 1$ infinitely often. Since the series contains infinitely many terms $\ge 1$, its limit is either $+\infty$ or nonexistent.
\end{proof}

\begin{cor_above}
	\label{from_below}
	If $U$ is computably unbounded from below on $D$, then\newline $E\left(U\left(x_1^n \Psi\left(Q, p, y_1^n, x_1^n\right)\right) | \Gamma_Q(y_1^n) = x_1^n \right)$ is either undefined or $-\infty$.
\end{cor_above}
\begin{proof}
	By definition, $-U$ is computably unbounded from above. Thus, by theorem \ref{big_thm}, $E\left(-U\left(x_1^n \Psi\left(Q, p, y_1^n, x_1^n\right)\right) | \Gamma_Q(y_1^n) = x_1^n \right)$ is either undefined or $+\infty$. So $E\left(U\left(x_1^n \Psi\left(Q, p, y_1^n, x_1^n\right)\right) | \Gamma_Q(y_1^n)\right) = -E\left(-U\left(x_1^n \Psi\left(Q, p, y_1^n, x_1^n\right)\right) | \Gamma_Q(y_1^n) = x_1^n \right)$ is either undefined or $-\infty$.
\end{proof}

\begin{cor_both}
	If $U$ is computably unbounded from both below and above on $D$, then $E\left( U\left(x_1^n \Psi\left(Q, p, y_1^n, x_1^n\right)\right) | \Gamma_Q(y_1^n) = x_1^n \right)$ is undefined.
\end{cor_both}
\begin{proof}
	By theorem \ref{big_thm}, $E\left( U\left(x_1^n \Psi\left(Q, p, y_1^n, x_1^n\right)\right) | \Gamma_Q(y_1^n) = x_1^n \right)$ is either undefined or $+\infty$. By corollary \ref{from_below}, it is either undefined or $-\infty$. Thus it is undefined.
\end{proof}

\section{Discussion}

Our main result implies that if you have an unbounded, perception determined, computable utility function, and you use a Solomonoff-like prior (Solomonoff, 1964), then you have no way to choose between policies using expected utility. So which of these things should we change?

We could use a non-perception determined utility function. Then our main result would not apply. In this case, the existence of bounded expected utility will depend on the utility function. It may be possible to generalize our argument to some larger class of utility functions which have a different domain. 

We could use an uncomputable utility function. For instance, if the utility of any perception-sequence is defined as equal to its Kolmogorov complexity, then the utility function is unbounded but the expected utility of any policy is finite.

We could use a smaller hypothesis space; perhaps not all computable environments should be considered.

The simplest approach may be to use a bounded utility function. Then convergence is guaranteed.

\end{document}